\documentclass{article}

\usepackage{microtype}
\usepackage{graphicx}
\usepackage{subcaption}
\usepackage{booktabs}
\usepackage{makecell}
\usepackage{amsmath}
\usepackage{amssymb}
\usepackage{mathtools}
\usepackage[noend]{algpseudocode}
\makeatletter
\@namedef{ver@algorithmic.sty}{}
\makeatother
\usepackage{amsthm}
\usepackage{multicol}
\usepackage{multirow}
\usepackage{inconsolata}
\usepackage{wrapfig}
\usepackage{colortbl}
\usepackage[table]{xcolor}
\usepackage{xspace}
\usepackage{tabularx}
\usepackage{pifont}
\usepackage{etoolbox}
\usepackage{enumitem}

\usepackage[accepted]{icml2026}
\usepackage{xurl}
\usepackage[hidelinks]{hyperref}
\newcommand{\paraf}[1]{\noindent\textbf{#1}}

\newcommand{\sysname}{\textsc{ToolPro}\xspace}

\definecolor{myyellow}{rgb}{.99,.94,.82}

\usepackage[capitalize,noabbrev]{cleveref}

\theoremstyle{plain}
\newtheorem{theorem}{Theorem}[section]
\newtheorem{proposition}[theorem]{Proposition}

\theoremstyle{definition}

\theoremstyle{remark}

\usepackage[textsize=tiny]{todonotes}

\icmltitlerunning{Beyond Static Endpoints: Tool Programs as an Interface for Flexible Agentic Web Services}

\begin{document}

\twocolumn[
  \icmltitle{Beyond Static Endpoints: Tool Programs as an Interface \\for Flexible Agentic Web Services}




  \begin{icmlauthorlist}
      \icmlauthor{Mugeng Liu}{cspku}
      \icmlauthor{Shuoqi Li}{smpku}
      \icmlauthor{Yixuan Zhang}{cspku}
      \icmlauthor{Yun Ma}{aipku}
    % \icmlauthor{Firstname1 Lastname1}{equal,yyy}
    % \icmlauthor{Firstname2 Lastname2}{equal,yyy,comp}
    % \icmlauthor{Firstname3 Lastname3}{comp}
  \end{icmlauthorlist}

  \icmlaffiliation{cspku}{School of Computer Science, Peking University, Beijing, China}

  \icmlaffiliation{smpku}{School of Software \& Microelectronics, Peking University, Beijing, China}

  \icmlaffiliation{aipku}{Institute for Artificial Intelligence, Peking University, Beijing, China}


  \icmlcorrespondingauthor{Yun Ma}{mayun@pku.edu.cn}

  \icmlkeywords{Agentic Web, Service Interface, Tool Use}

  \vskip 0.3in
]




\printAffiliationsAndNotice{}  % no special notice (required even if empty)

\begin{abstract}

In the agentic web era, LLM-based agents increasingly invoke web services as tools, yet most interfaces remain \emph{static endpoints} that poorly express long-horizon workflows with loops, conditionals, joins, and retries. We present \sysname, which represents an agent's tool intent as an \emph{executable tool program} that compactly encodes multi-step service interactions with explicit effect types. \sysname combines constraint-guided program construction, effect-aware replay for exactly-once state-modifying calls, and a profile-driven policy that decides when program execution outperforms stepwise calling. We instantiate \sysname over MCP-style services with WebAssembly sandboxing and evaluate it on diverse workflows of real-world applications. \sysname reduces end-to-end latency by up to 53.4\% and client-side traffic by up to 96.1\%, with larger gains under higher network latency and workflow complexity.

\end{abstract}

\section{Introduction}

LLM-based agents~\cite{yao2023react,schick2023toolformer,liu2024agentbench,qin2024toolllm,liu2026roga} are increasingly expected to complete long-horizon workflows by orchestrating web services.
Yet most services are still exposed through \emph{static API endpoints}---an interface designed for single-shot queries, not for procedural, multi-step interaction.
When a task requires control flow (e.g., loops, conditionals), intermediate bindings, or intent-dependent data access, an agent must externalize the workflow into a brittle sequence of endpoint calls interleaved with multi-round reasoning~\cite{yao2022webshop,deng2023mind2web,zhou2024webarena}.
The stepwise interface scales poorly, multiplying network turns, systematically over- and under-fetching data, and triggering cascading retries with inconsistent side effects upon partial failure.

Our insight is that the inefficiency is fundamentally \emph{representational}, as shown in~\autoref{fig:motivation}.
Endpoint sequences are a weak interface for expressing tool intent, because they fragment a coherent multi-step plan into local decisions conditioned on intermediate responses.
As a result, both client--service round trips and agent reasoning rounds grow with the number of procedural steps, while failures amplify as a single mismatch can cascade into retries and state inconsistencies.
Agentic workflows need an interface that can express \emph{``perform this multi-step interaction''} as a single, composable object whose execution can be delegated, optimized, and checked.

\begin{figure*}[t!]
\centering
\includegraphics[width=0.9\linewidth]{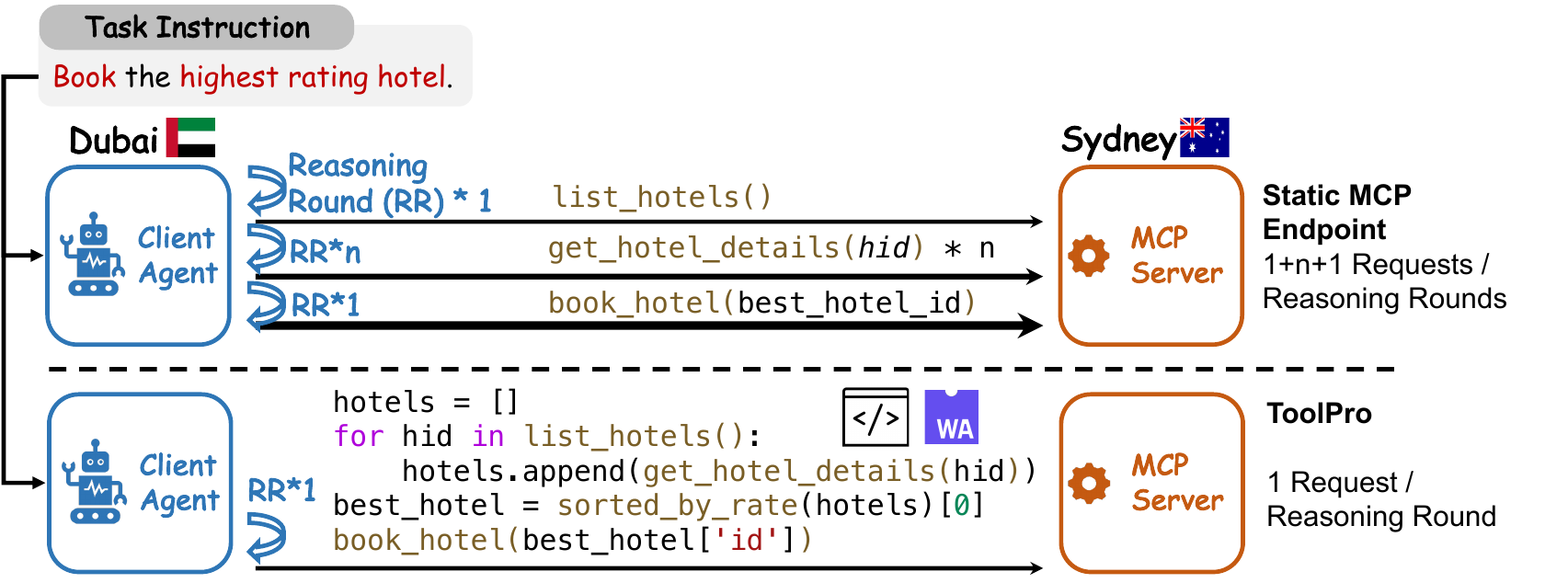}
% \vspace{-0.6em}
\caption{\textbf{From static endpoints to tool programs.}
Stepwise endpoints force the agent to repeatedly call endpoints and re-prompt to realize control flow.
Tool programs (\sysname) instead package a multi-step interaction as one executable object with explicit effects, enabling service-side execution and safe re-execution under repair.}
\label{fig:motivation}
\vspace{-1em}
\end{figure*}

To this end, we propose \emph{tool programs} as an executable representation of tool intent. These programs compactly encode a multi-step service interaction, complete with control flow and intermediate bindings. Furthermore, they feature explicit effect types that distinguish state-preserving READ operations from state-modifying WRITE operations.
Rather than repeatedly selecting endpoints on the client, an agent synthesizes a tool program and delegates its execution to a controlled service-side runtime.
Making intent first-class enables three capabilities that static endpoints do not provide, including
(i) \emph{turn reduction} by consolidating multi-step interactions into fewer network round trips,
(ii) \emph{effect-aware execution} by enforcing different semantics for READ versus WRITE, and
(iii) \emph{safe re-execution} under repair without duplicating side effects.
Realizing tool programs in practice raises three challenges:
(1) Executability. LLM-produced programs may fail to compile or crash at runtime under a typed, sandboxed substrate.
(2) Side effects under repair. Partial success before failure makes naive re-execution duplicate WRITEs and corrupt service state.
(3) When to consolidate. For short workflows or low-latency networks, program construction overhead can outweigh the savings from turn reduction.

By addressing these challenges, we present \sysname, the first agentic web service runtime that operationalizes \textbf{Tool} \textbf{Pro}grams as an interface for flexible agentic web services.
Specifically, to improve executability, \sysname introduces constraint-guided program construction, which combines lightweight formatting constraints with compiler/runtime feedback, and resolves common failures via service-side repair to avoid repeated client--server back-and-forth.
To control side effects, \sysname enforces effect-aware replay that provides \emph{exactly-once} semantics for WRITE operations across iterative repair and re-execution.
To decide when consolidation is worthwhile, \sysname applies a profile-driven consolidation rule that adaptively selects between stepwise calling and program execution.

We implement \sysname over MCP-style web services using WebAssembly sandboxing and evaluate it on diverse workflows within realistic applications. 
\sysname reduces end-to-end latency by up to 53.4\% and client-side traffic by up to 96.1\%, with gains increasing under higher network latency and workflow complexity.
These results highlight that \sysname lays an important foundation for agent-facing service interfaces in the emerging agentic web.

This paper makes the following contributions.\footnote{Code is publicly available at \url{https://github.com/morgen52/toolpro_icml26}.}
% \begin{itemize}[leftmargin=*,topsep=0pt,itemsep=0pt]
\begin{itemize}[leftmargin=*,topsep=2pt,itemsep=2pt,parsep=0pt,partopsep=0pt]
\item We identify static endpoints as a representational bottleneck for agentic web workflows and propose \emph{tool programs} as an agent-facing service interface.
\item We present \sysname to make tool programs practical by addressing core challenges. It employs constraint-guided construction for executability, effect-aware replay with exactly-once semantics for safe repairs, and a profile-driven policy for adaptive consolidation.
\item We evaluate \sysname on real applications and workflows, demonstrating substantial reductions in latency and client-side traffic, highlighting a promising direction to build an efficient agentic web.
\end{itemize}

\section{Problem Formulation}
\label{sec:problem}

We study agentic web-service tool use where a client-side LLM agent orchestrates server-side web services to complete procedural workflows.
A tool-facing service exposes a set of endpoints $\mathcal{E}$ over an internal service state $s$.
A call is $\textsc{Call}(e,a)$ for $e\in\mathcal{E}$ and arguments $a$, returning an output $o$ (or an error) and possibly updating $s$.
We focus on procedural \emph{agentic workflows} that inherently require control flow, intermediate bindings, and intent-dependent data access.

\textbf{Bottleneck: stepwise endpoint sequences.}
With \emph{static endpoints}, an agent realizes an intent via a stepwise interaction loop.
At step $i$, it decides the next call $(e_i,a_i)$ conditioned on the task context and past observations $(o_1,\ldots,o_{i-1})$, then issues the request and observes $o_i$.
This yields a call sequence $\pi=\langle(e_1,a_1),\ldots,(e_N,a_N)\rangle$ interleaving $N$ client--service round trips with $N$ client-side decision rounds.
This interface makes the procedure \emph{reactive}, which fragments a coherent procedure into client-side next-call decisions.
As procedural length grows, this interface (i) inflates latency via repeated RTT and per-step decision overhead, (ii) induces over-/under-fetching because control logic must be implemented outside the service, and (iii) makes recovery brittle, with partial failures triggering retries that can duplicate state-modifying operations and corrupt state.

\textbf{Key idea: tool programs as an interface.}
We propose to make tool intent first-class by representing a workflow as an \emph{executable tool program}.
A tool program $P$ is a procedure whose atomic operations are endpoint invocations $\textsc{Call}(e,a)$, composed with structured control flow and intermediate bindings.
Given an initial state $s$, executing $P$ produces a result $y$ and a (possibly updated) state $s'$.
This shifts the interface from next-call selection to program submission, enabling service-side execution, optimization, and checking while preserving the service's observable behavior.

\textbf{Goals.}
We aim to make tool programs practical \emph{without} changing the observable outcomes of the underlying service interaction and \emph{without} introducing prohibitive overhead compared to stepwise calling.
Concretely, we target
(G1) \emph{effect-safe observable semantics} under failures and retries; and
(G2) \emph{improved end-to-end efficiency} in latency and client-side traffic.

\emph{(G1)}
Executing a tool program may fail (e.g., due to runtime errors in generated code), triggering repair and re-execution.
We aim for an \emph{interface-level} guarantee: re-execution for the same high-level intent must not introduce additional side effects beyond stepwise execution.
Note that we do not attempt to eliminate endpoint-level failures, which are inherent to the underlying service and can equally occur under stepwise execution.

To reason about side effects, we assume each endpoint has an effect label $\mathsf{eff}(e)\in\{\text{READ},\text{WRITE}\}$, distinguishing state-preserving queries from state-modifying operations.
An execution induces a trace $\tau(P)=\langle(e_1,a_1,o_1),\ldots,(e_N,a_N,o_N)\rangle$, where only \text{WRITE} calls may change $s$.
We target two interface-level properties:
(i) \emph{observational equivalence}: conditioned on the same underlying sequence of endpoint outcomes, program execution exposes the same outputs/errors as stepwise execution; and
(ii) \emph{retry safety}: across repair-driven re-executions for the same intent, \text{WRITE} effects are not duplicated, yielding interface-level \emph{exactly-once} semantics.

\emph{(G2)}
Tool programs are beneficial only when their one-time construction cost is amortized by reducing round trips.
A simple latency model contrasts stepwise calling and program execution, as follows:
\begin{equation}
\label{eq:cost}
\textstyle
\begin{aligned}
T_{\textsc{step}} &\approx \sum_{i=1}^{N}\bigl(T_{\textsc{rtt}}+T_{\textsc{dec}}+T_{\textsc{api}}\bigr), \\
T_{\textsc{prog}} &\approx T_{\textsc{build}} + \Bigl(T_{\textsc{rtt}}+\sum_{i=1}^{N}T_{\textsc{api}}\Bigr),
\end{aligned}
\end{equation}
where $T_{\textsc{rtt}}$ is client--service RTT, $T_{\textsc{dec}}$ is per-step decision overhead on the client-side agent, $T_{\textsc{api}}$ is endpoint execution time, and $T_{\textsc{build}}$ includes program construction, compilation, and potential repair.
Client-side traffic follows a similar trade-off. 
Stepwise calling repeatedly transmits requests and prompts across $N$ rounds, whereas tool programs consolidate interaction into a small number of uploads and responses.

\textbf{Challenges.}
This formulation surfaces three challenges in operationalizing tool programs in practice:
(C1) \emph{executability} of LLM-produced programs under a typed, sandboxed substrate;
(C2) \emph{exactly-once effect semantics} for \text{WRITE} calls under repair and re-execution; and
(C3) \emph{adaptive consolidation} to decide when program execution is more efficient than stepwise calling.

\section{\sysname Design}

\sysname makes \emph{tool programs} a first-class service interface for agentic workflows.
Given an intent instance and tool specifications, the agent submits a single effect-typed program $P$; the service-side runtime then compiles, optionally repairs, and executes $P$ as a unit while enforcing an interface contract.
\sysname achieves this interface shift by addressing the three challenges (introduced in \S\ref{sec:problem}) in operationalizing tool programs in practice.

\sysname follows a \emph{synthesize--project--compile--execute} pipeline with a conservative fallback path.
(1) Given task intent and service endpoints, the client synthesizes a candidate tool program and performs lightweight structural checks to reject obviously misaligned programs early.
(2) The server projects the program into a constrained interface-program surface that is analyzable and enforceable, then compiles and executes it in a sandbox.
(3) If compilation or execution fails, the server performs bounded in-place repair using compiler diagnostics and runtime traces as verifiable feedback.
During any execution and re-execution, the runtime mediates external calls and enforces effect-aware replay to prevent duplicated \text{WRITE} effects.
Finally, a profile-driven consolidation policy decides whether to use program execution or revert to stepwise calling for the current task.

\subsection{Tool Programs as an Interface}

\sysname employs the \emph{tool program} as the submitted interface object and the contract that the runtime enforces for any accepted program.

The tool program serves as the interface object, which is the unit of interaction at the interface.
A client represents an agentic workflow as a tool program $P$ (defined in \S\ref{sec:problem}) and submits it to the runtime for execution as a unit.
The only way for $P$ to interact with the underlying service is via a unified stub $\textsc{Call}(e,a)$, where each dynamic call (a runtime instance of $\textsc{Call}$) triggers one endpoint invocation and returns an output (or error) observable to the program.
This makes the workflow logic explicit and inspectable on the server side.

To make such a program $P$ enforceable and safe as an interface, the runtime cannot accept arbitrary code, which motivates a constrained surface.
\sysname restricts submitted programs to a constrained surface that expresses service-interaction logic rather than arbitrary application logic.
Concretely, the runtime enforces: (i) structured control flow only (\texttt{if}/\texttt{else}, \texttt{for}/\texttt{while}); (ii) external interaction only through $\textsc{Call}(\cdot)$ with explicit effect annotations at each call site; (iii) no exceptions, threads, dynamic linking, or unsafe memory operations; and (iv) no ad-hoc networking or filesystem access beyond the tool-facing service boundary.
These restrictions ensure that compilation errors, runtime traces, and call events refer to a stable, analyzable surface that the runtime can repair and mediate reliably.

Moreover, \sysname makes side effects explicit via effect typing. Concretely, every external call site in $P$ must declare an effect label (\text{READ}/\text{WRITE}), which the runtime checks against the endpoint’s declared effect $\mathsf{eff}(e)$. 
These annotations serve as the handle for retry protection. \text{READ} calls may be safely re-issued, whereas completed \text{WRITE} calls must be replay-protected under repair-driven re-execution.

Building on the interface object, constrained surface, and effect-typed boundaries, \sysname enforces the following contract for any well-formed program $P$.
\begin{itemize}[leftmargin=*,topsep=0pt,itemsep=0pt]
\item \textbf{Program order.} \sysname processes dynamic calls in program order and never reorders external invocations.
\item \textbf{Observable preservation.} Conditioned on identical endpoint outcomes (including inherent service errors), \sysname exposes the same per-call outputs/errors as stepwise execution. Across repair-driven re-executions for the same intent instance, completed \text{WRITE} effects are not duplicated at the service boundary.
\item \textbf{Safe fallback.} If construction or repair exceeds the attempt budget, violates constraints, or encounters unsupported \text{WRITE} semantics, \sysname falls back to stepwise endpoint calling and surfaces diagnostics.
\end{itemize}
The contract is realized by three mechanisms, each addressing one of the three challenges (introduced in \S\ref{sec:problem}), as shown in \autoref{fig:overview}.

\begin{figure}[t!]
\centering
\includegraphics[width=0.45\textwidth]{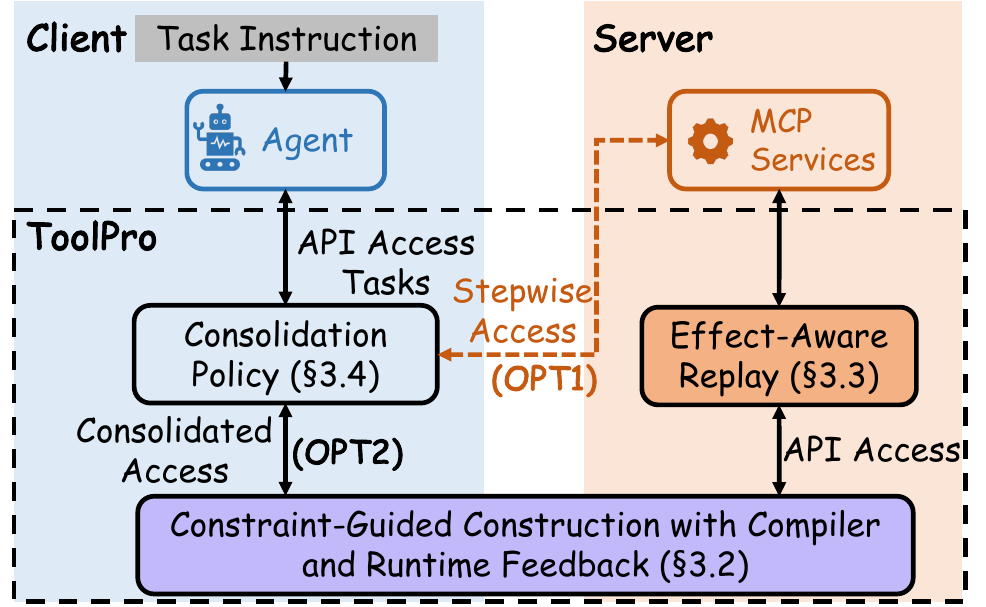}
% \vspace{-1em}
\caption{Three mechanisms of \sysname.}
\label{fig:overview}
% \vspace{-1em}
\end{figure}

\subsection{Constraint-Guided Construction with Compiler and Runtime Feedback}

The first challenge to realize tool programs is to ensure reliable program execution (C1).
In practice, LLM-produced programs are often structurally plausible yet fail under a typed sandbox due to unsupported libraries, type mismatches, missing imports, or brittle error handling.
\sysname addresses this gap with a constraint-guided construction pipeline with four steps, which turns failures into bounded, checkable repairs using compiler and runtime feedback.

\paraf{Step 1: client-side synthesis with lightweight interface checks.}
Given the task intent and tool specifications, the client synthesizes a candidate program $P$ and performs conservative checks that are cheap yet effective at filtering obviously misaligned candidates.
Concretely, it verifies (i) endpoint coverage (required endpoints appear), (ii) a control-flow skeleton consistent with the intent (e.g., a bounded loop and necessary conditionals), and (iii) basic value-flow sanity (outputs of earlier \text{READ} calls can be bound and used as inputs to later calls).
These checks avoid sending clearly ill-formed programs into server-side compilation.

\paraf{Step 2: server-side projection into a constrained surface.}
Upon receiving $P$, the server applies a deterministic projection $\Pi(\cdot)$ to rewrite it into the constrained interface-program surface.
The projection (i) rewrites all external interactions to the unified stub $\textsc{Call}(\cdot)$ with explicit \text{READ}/\text{WRITE} annotations, (ii) removes or replaces unsupported imports and patterns, and (iii) rejects any syntax or language features outside the allowed surface.
The result, a canonical form $\Pi(P)$, ensures that subsequent diagnostics and traces are expressed over a stable surface that the runtime can interpret, enforce, and modify reliably.

\paraf{Step 3: compile and execute with feedback-driven, bounded in-place repair.}
The server then compiles and executes $\Pi(P)$ in a sandbox and uses the resulting feedback to repair and re-run within a fixed attempt budget.
On compilation failure, the compiler returns precise diagnostics that localize missing symbols, type mismatches, and offending code spans.
On runtime failure, the runtime records a lightweight trace that identifies the failing region together with the prefix of dynamic calls already observed.
\sysname applies \emph{in-place} repairs conditioned on this verifiable feedback and re-enters the loop.
It re-projects the revised program if needed, recompiles, and re-executes until success or the budget is exhausted.
Repairs are localized whenever possible.
Invocation errors trigger edits at the relevant call site (endpoint, arguments, or effect annotation), while control-flow/logic errors trigger rewrites of the minimal affected block with the rest preserved.

\paraf{Step 4: Safe fallback with diagnostics.}
If the attempt budget is exceeded or a repair violates the constrained surface, \sysname falls back to stepwise endpoint calling and surfaces the diagnostics collected during steps.
This ensures that program execution is attempted only when it is demonstrably viable, and that failures are surfaced with auditable diagnostics rather than silently ignored.

\subsection{Effect-Aware Replay for Retry-Safe Semantics}

The second challenge (C2) is retry safety under repair-driven re-execution.
A tool program may partially succeed before failing. 
Naively re-running repaired code can re-emit state-modifying \text{WRITE} calls, duplicating side effects and corrupting service state.
\sysname prevents this by mediating runtime calls at the interface boundary and replaying outcomes of completed \text{WRITE} calls across re-executions.

Because retry safety must hold at runtime, \sysname intercepts each dynamic call (a runtime instance of $\textsc{Call}(e,a)$) along the taken control-flow path.
\text{READ} calls are always forwarded to the service, since re-issuing them does not introduce new side effects.
In contrast, \text{WRITE} calls are replay-protected.
Once a \text{WRITE} completes, subsequent re-executions must not re-emit it, while the program continues to observe the same outcome.

To support such replay, \sysname maintains a per-intent-instance log of committed \text{WRITE} outcomes.
Specifically, it keeps (i) an ordered history log $\mathcal{H}$ containing \text{WRITE} calls completed in prior executions, and (ii) a working log $\mathcal{W}$ for \text{WRITE} calls completed in the current execution.
Each entry stores $(e,a,o)$ along with a per-re-execution flag \texttt{used}.
Both logs are scoped to a single intent instance, ensuring replay does not leak across unrelated tasks.

Since replay is only sound when repairs do not attempt to revise already committed effects, \sysname enforces a conservative replay discipline.
If a repaired program changes the arguments or relative order of any committed \text{WRITE} prefix already recorded in $\mathcal{H}$, \sysname disables replay and falls back to stepwise calling.
This turns a potential silent semantic divergence into an explicit, auditable condition.

Under this discipline, replay reduces to a simple matching problem at each \text{WRITE}.
On re-execution, when the program reaches the next dynamic \text{WRITE} with parameters $(e,a)$, \sysname matches it to the earliest unused entry in $\mathcal{H}$ with the same $(e,a)$.
If a match exists, the runtime returns the cached outcome $o$ without issuing the external call; otherwise, it emits the call, obtains $o$, and appends $(e,a,o)$ to $\mathcal{W}$.
Ordered matching treats repeated writes with identical $(e,a)$ as distinct dynamic calls within a run, while still preventing duplication across re-executions.

The log-based replay mechanism is shown in Algorithm~\ref{alg:replay}.
Before each re-execution, the runtime archives the \text{WRITE} calls completed in the last execution by appending $\mathcal{W}$ to $\mathcal{H}$, clears $\mathcal{W}$, and resets all \texttt{used} flags.
\text{READ} calls are always emitted, whereas \text{WRITE} calls are either replayed from $\mathcal{H}$ (if matched) or emitted and recorded into $\mathcal{W}$.

\begin{algorithm}[t!]
\caption{Effect-Aware Replay for Retry-Safe Semantics}
\label{alg:replay}
\begin{algorithmic}[1]
\Require $\mathcal{H}$: ordered history of completed \text{WRITE} calls across prior executions
\Require $\mathcal{W}$: ordered log of completed \text{WRITE} calls in the current execution
\Statex
\Procedure{OnReExecution}{}
    % \State $\mathcal{H} \gets \mathcal{H} \cup \mathcal{W}$ \Comment{Archive last execution}

    \State $\mathcal{H} \gets \mathrm{Concat}(\mathcal{H}, \mathcal{W})$ \Comment{Append last execution in order}
    \State $\mathcal{W} \gets \emptyset$
    \For{\textbf{each} entry $x \in \mathcal{H}$}
        \State $x.\texttt{used} \gets \textbf{False}$
    \EndFor
\EndProcedure
\Statex
\Function{HandleCall}{$(e,a,\textsf{eff})$}
    \If{$\textsf{eff}=\text{READ}$}
        \State \Return \Call{Emit}{$(e,a)$}
    \Else \Comment{$\textsf{eff}=\text{WRITE}$}
        \For{\textbf{each} entry $x \in \mathcal{H}$ \textbf{in order}}
            \If{$x.(e,a)=(e,a)$ \textbf{and} $\neg x.\texttt{used}$}
                \State $x.\texttt{used} \gets \textbf{True}$
                \State \Return $x.o$ \Comment{Replay outcome}
            \EndIf
        \EndFor
        \State $o \gets \Call{Emit}{(e,a)}$
        \State $\mathcal{W}.\texttt{append}((e,a,o))$
        \State \Return $o$
    \EndIf
\EndFunction
\end{algorithmic}
\end{algorithm}

\begin{proposition}[Retry-safe \text{WRITE} emissions]
\label{prop:exactly-once}
Fix an intent instance and a well-formed interface program $P$.
Across repair-driven re-executions that satisfy the replay discipline above, \sysname emits each completed dynamic \text{WRITE} call to the underlying service at most once, and later re-executions replay its cached outcome.
\end{proposition}
\noindent\textit{Proof.}
The only emission of a \text{WRITE} occurs when no unused matching entry exists in $\mathcal{H}$ (Lines 9--16), in which case the emitted outcome is recorded and later archived into $\mathcal{H}$.
In subsequent re-executions, the same dynamic \text{WRITE} must match an unused entry and be replayed (lines~10--13), so it cannot be emitted again.
\hfill$\square$

Moreover, practical services often provide idempotency keys or exhibit nondeterminism, which affects replay matching.
When idempotency keys are available, \sysname includes them in argument $a$, strengthening matching and aligning replay with the service’s own exactly-once intent.
If an endpoint is meaningfully nondeterministic under identical $(e,a)$ and no idempotency is available, \sysname conservatively falls back to stepwise calling, as replay could otherwise lead to different results compared to stepwise execution.

\subsection{Profile-Driven Consolidation Policy}

The remaining challenge (C3) is \emph{when} to pay the one-time build cost of program execution.
While tool programs can reduce client-server turns, their construction (synthesis, projection, compilation, and possible repair) is nontrivial, so consolidation should be used only when it is predicted to reduce end-to-end cost.
\sysname therefore uses lightweight online profiling and decision rules to choose between the tool program and stepwise calling.

To make this choice instance-adaptive, \sysname maintains moving averages from recent runs for $\overline{T_{\textsc{rtt}}}$, $\overline{T_{\textsc{dec}}}$, and $\overline{T_{\textsc{build}}}$.
Here $T_{\textsc{rtt}}$ is the client--service round-trip time, $T_{\textsc{dec}}$ is per-step client-side decision overhead, and $T_{\textsc{build}}$ includes program synthesis, projection, compilation, and any repair time.
In addition, \sysname estimates $N$, the number of dynamic endpoint invocations in the candidate program, using synthesized structure and loop bounds when available.

Given these estimates, the policy follows the cost model in \autoref{eq:cost}.
Consolidation primarily saves $(N\!-\!1)$ additional RTTs and decision rounds, while endpoint execution time largely appears in both modes.
\sysname predicts the net benefit of program execution over stepwise calling as
\begin{equation}
\label{eq:benefit}
\Delta T \;=\; (N-1)\cdot(\overline{T_{\textsc{rtt}}}+\overline{T_{\textsc{dec}}}) \;-\; \overline{T_{\textsc{build}}}.
\end{equation}
If $\Delta T>0$, \sysname executes the tool program; otherwise it selects stepwise calling.

Because profiles may be inaccurate at cold start, \sysname bootstraps with a small number of stepwise runs to initialize $\overline{T_{\textsc{rtt}}}$ and $\overline{T_{\textsc{dec}}}$, and enables tool program execution only when the synthesized structure is clearly multi-step (e.g., a bounded loop).
Once estimates stabilize, \autoref{eq:benefit} is applied to make per-instance decisions.

This policy integrates naturally with the end-to-end control flow.
Given an intent, \sysname synthesizes a candidate program, estimates $\Delta T$, and selects the mode.
If program execution is chosen, the server runs projection, compilation, and bounded in-place repair under effect-aware execution. 
If repair exceeds the attempt budget, violates constraints, or encounters unsupported \text{WRITE} semantics, \sysname falls back to stepwise calling with diagnostics.
As a result, \sysname uses tool programs only when they are predicted to be beneficial in efficiency.

\section{Experiments}

\subsection{Implementation}

We instantiate \sysname over MCP-style tool-facing services and implement the service-side runtime using WebAssembly (Wasm).
On the client, an LLM synthesizes a tool program $P$ from the intent and tool specifications and applies lightweight interface checks; the consolidation policy selects between program mode and stepwise calling.
On the server, \sysname projects $P$ into the constrained interface-program surface via $\Pi(\cdot)$, compiles and executes it in a sandbox with bounded in-place repair, and enforces retry-safe \text{WRITE} semantics by mediating every dynamic \textsc{Call} and replaying completed \text{WRITE} outcomes across re-executions.
If projection/repair violates constraints or exceeds the attempt budget, \sysname falls back to stepwise calling with diagnostics.

We use Wasm as the execution substrate because it provides (i) a strong sandbox boundary for untrusted, LLM-generated code with no ambient authority, (ii) a capability-style host interface that lets us expose only the unified $\textsc{Call}(e,a,\mathsf{eff})$ stub and mediate all side effects, and (iii) portable, low-overhead execution suitable for short-lived procedural workloads.
Additional implementation details appear in Appendix~\ref{app:impl}.

\subsection{Experimental Setup}

\paraf{Benchmarks.}
We evaluate \sysname on three widely-used open-source applications (Memos, Directus, and MinIO) from GitHub, following prior studies~\cite{gu2025orfa}.
Details on these applications are shown in Appendix \ref{ap:benchmarks}.
Each application is exposed as an MCP-style tool-facing service with fixed endpoints, and we construct procedural workflows that require loops/conditionals and intermediate bindings.

For each application, we construct two workflows: a read-only workflow (suffix \texttt{.r}) and a read-write workflow (suffix \texttt{.w}).
Each workflow is parameterized by $N$, which controls procedural length (and equals the number of endpoint invocations in stepwise execution).
More details on the constructed workflows are shown in Appendix~\ref{ap:workflow}.
We also add supplemental realistic workflows in Appendix~\ref{ap:supplemental_workflows} to stress nondeterministic retrieval, branching, coordinated writes, non-idempotent effects, and cross-service execution.

\paraf{Metrics.}
We report end-to-end latency and client-side traffic volume, two widely used metrics~\cite{gu2025orfa}.
Latency measures the time to complete a workflow, including client-side LLM time, client-server communication, and server-side execution.
Client-side traffic measures bytes transmitted by the client during workflow execution, including both client-server payloads and client-to-LLM prompts.
For supplemental workflows, we also report task accuracy.
% \draft{For the supplemental workflows, we additionally report task accuracy because these tasks include runtime binding and state-consistency conditions.}

\paraf{Baselines and variants.}
We compare \sysname against the prevailing stepwise endpoint interface and include two \sysname variants to isolate the effect of consolidation.

\begin{itemize}[left=0pt,noitemsep,topsep=0pt,partopsep=0pt]
\item \textbf{Stepwise MCP Web Service (MWS).}
The prevailing fully stepwise baseline. At each step, the agent replans the next endpoint call $(e_i,a_i)$ given the intent, tool specs, and prior observations, leading to \textbf{$N$} requests and typically \textbf{$N$} LLM decision rounds with repeated tool-context transmission.

\item \textbf{\sysname-step.}
A stepwise \sysname variant (no consolidation) that replaces replanning with intent-structured guidance: it first derives a coarse call skeleton from the intent (e.g., required endpoints and loop/conditional structure), and then per step only instantiates/validates the next call using current observations (e.g., filling arguments from prior \text{READ} outputs), still incurring \textbf{$N$} requests/rounds but with lower per-step decision overhead and less repeated tool context.

\item \textbf{\sysname-prog.} \sysname-prog always executes in program mode. The agent submits one effect-typed tool program, which the server projects ($\Pi(\cdot)$), compiles/repairs, and executes under call interception.
This isolates the benefits of consolidation (turn reduction) while paying the one-time build cost.

\item \textbf{\sysname.} Full system with the profile-driven consolidation policy, selecting between program mode and stepwise mode per instance.
\end{itemize}

Details on experimental environments are shown in Appendix~\ref{ap:env}.

\begin{figure}[t!]
\begin{center}
\includegraphics[width=0.49\textwidth]{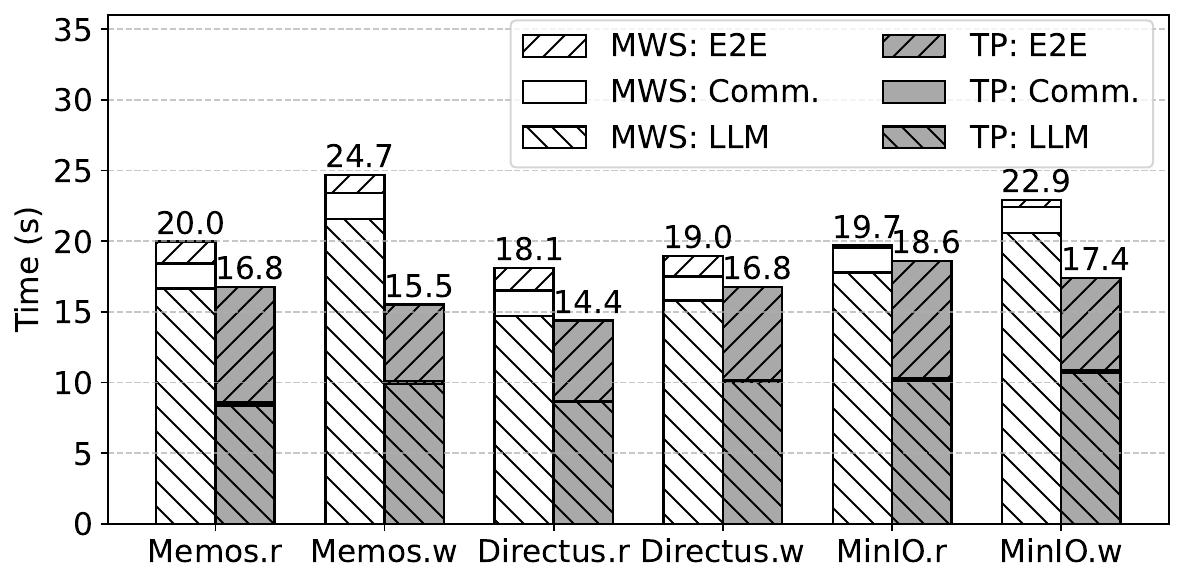}
\vspace{-2em}
\caption{Latency comparison. `TP' indicates \sysname. `E2E' is the end-to-end latency. `Comm.' is client-server communication latency. `LLM' is latency involving client-side LLM inference.}
\label{fig:latency}
\end{center}
\vspace{-0.8cm}
\end{figure}

\subsection{Overall Efficiency}

We first evaluate end-to-end efficiency by comparing \sysname against MWS across workflows with $N=10$.
Unless specified otherwise, the server is hosted in Sydney and the client in Beijing.

\paraf{Service latency.}
As shown in \autoref{fig:latency}, \sysname consistently reduces end-to-end latency across all workflows.
The main source of improvement is turn reduction: program mode packages a multi-step interaction into a single submission/execution cycle, avoiding the repeated client-side decide-next-call loop in MWS and reducing $(N-1)$ additional RTTs and reasoning rounds.
While \sysname incurs a one-time build cost (program synthesis, projection, compilation, and occasional bounded repair), this overhead is amortized for procedural workflows and higher RTT settings, consistent with the proposed policy model.
Latency improvements remain consistent for read-write workflows, indicating that enforcing retry-safe semantics does not dominate end-to-end cost.

\begin{figure}[t!]
\begin{center}
\includegraphics[width=0.49\textwidth]{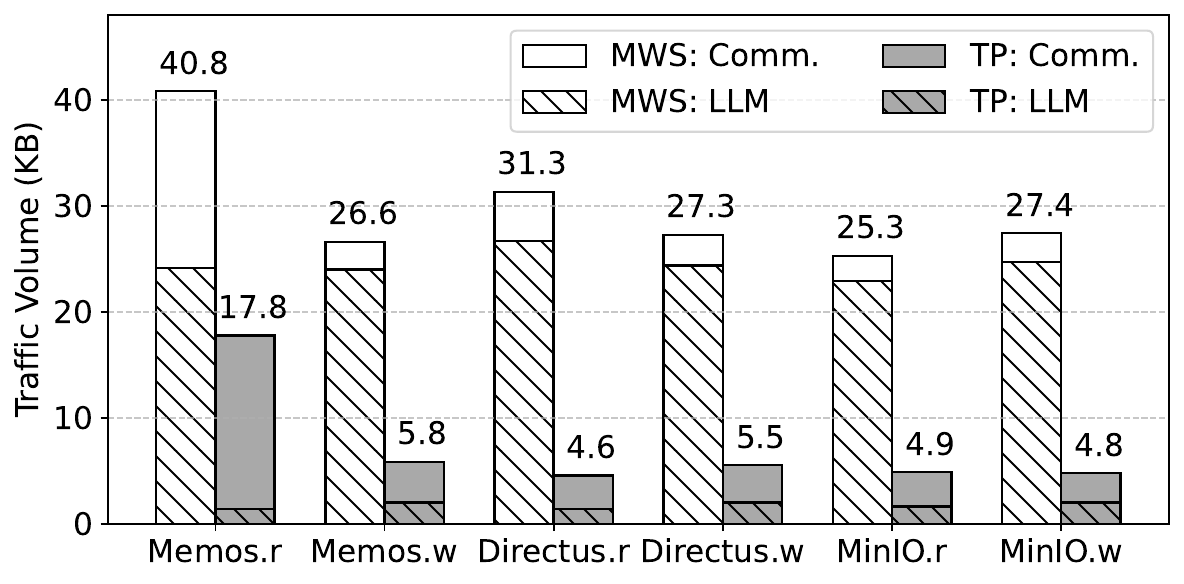}
\vspace{-2em}
\caption{Client-side traffic volume. `Comm.' denotes client-server bytes. `LLM' denotes client-to-LLM bytes.}
\label{fig:traffic}
\end{center}
\vspace{-0.8cm}
\end{figure}

\paraf{Client-side traffic volume.}
As shown in \autoref{fig:traffic}, \sysname reduces client-side traffic by up to 85.3\%.
This reduction is primarily driven by fewer client-to-LLM exchanges: MWS repeatedly transmits tool specifications and intermediate context across $N$ rounds, whereas program mode transmits a compact tool program once (plus bounded repair prompts when needed).
Client--server bytes may fluctuate because program mode uploads program source/bytecode, but the reduction in repeated prompting and over/under-fetching dominates for procedural workflows.

\paraf{Realistic workflows and reliability.}
% \draft{
To test whether the gains go beyond fixed-$N$ procedural loops, we add four supplemental complex benchmarks (cbench1--cbench4) that require nondeterministic retrieval, loop/branch logic, coordinated multi-record writes, non-idempotent side effects, and cross-service branching. Across cbench1--cbench3 with complex realistic workflows, \sysname reduces end-to-end latency from 30.16s to 17.91s (40.6\%), improves task accuracy from 0.60 to 0.93, and reduces client-side LLM latency from 14.98s to 7.08s (52.8\%). On the cross-service benchmark (cbench4), \sysname reduces latency from 52.68s to 24.54s (53.4\%), cuts client-side traffic by 96.1\%, and improves accuracy from 0.20 to 0.80.
%}

We observe that the program-mode failure rate is relevant to the coding capability of LLMs.
On the Rust-based cbench2, qwen3-coder-flash reaches 80\% success rate, while gpt-5.1 and gemini-3-flash-preview each reach 100\% success rate with no observed compilation failure or fallback. 
In addition, replay also matters operationally: over a 15-run no-replay ablation on cbench1--cbench3, disabling replay increases average latency from 17.92s to 21.45s (+19.7\%) and fallback from 0/15 to 3/15. Thus \sysname fails closed when replay safety cannot be guaranteed, and fallback does not erase the overall gains over complex realistic workflows.

\subsection{Sensitivity}

We next study how network conditions and workflow complexity affect performance, using \texttt{Memos.w}.

\paraf{Impact of network conditions.}
To vary network conditions, we deploy servers in London, Sydney, and San Francisco with a client in Beijing, and also inject additional one-way delays of 50ms, 100ms, 150ms, and 200ms.
We run with $N=10$.

As shown in \autoref{fig:network_condition}, \sysname outperforms MWS across network settings.
As RTT increases, the performance gap widens because consolidation saves $(N-1)$ additional round trips, and the policy increasingly favors program mode when the predicted benefit exceeds build cost.
In low-latency conditions, the policy more often selects stepwise mode to avoid paying the one-time build cost, yielding robust performance across conditions.
% \draft{
An extended 1--2000ms RTT sweep on \texttt{Memos.w} ($N=8$) makes the mode switch explicit. \sysname-step is better at 1--100ms (14.50--15.30s versus 15.51--15.60s for \sysname-prog), while \sysname-prog is better at 1000--2000ms (16.50--17.51s versus 22.48--30.50s for \sysname-step). The full \sysname policy stays close to the empirically better mode across the sweep, selecting stepwise execution in low-latency settings and program execution as RTT dominates.
%}

\paraf{Impact of workflow complexity.}
We vary workflow complexity by setting $N=5,10,15,20$ and compare MWS, \sysname-step, \sysname-prog, and \sysname, with the server hosted in Sydney and the client in Beijing.

\begin{figure}[t!]
\begin{center}
\includegraphics[width=0.49\textwidth]{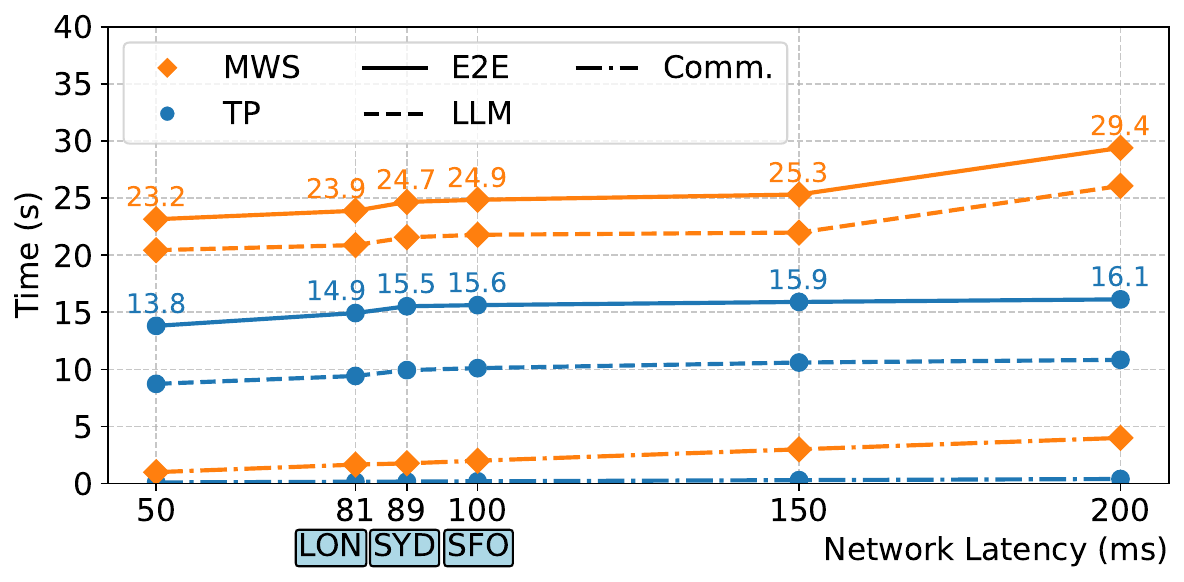}
\vspace{-2em}
\caption{Impact of varying network conditions.}
\label{fig:network_condition}
\end{center}
\vspace{-1em}
\end{figure}

\begin{figure}[t!]
\begin{center}
\includegraphics[width=0.49\textwidth]{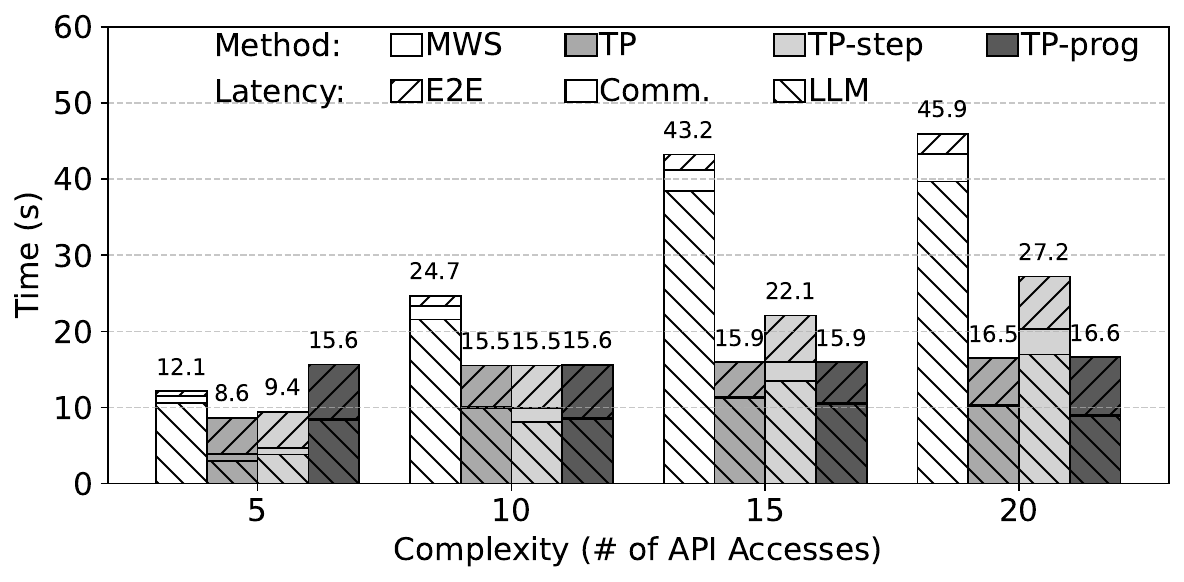}
\vspace{-2em}
\caption{Impact of workflow complexity on latency.}
\label{fig:complexity_latency}
\end{center}
\vspace{-1em}
\end{figure}

\begin{figure}[t!]
\begin{center}
\includegraphics[width=0.49\textwidth]{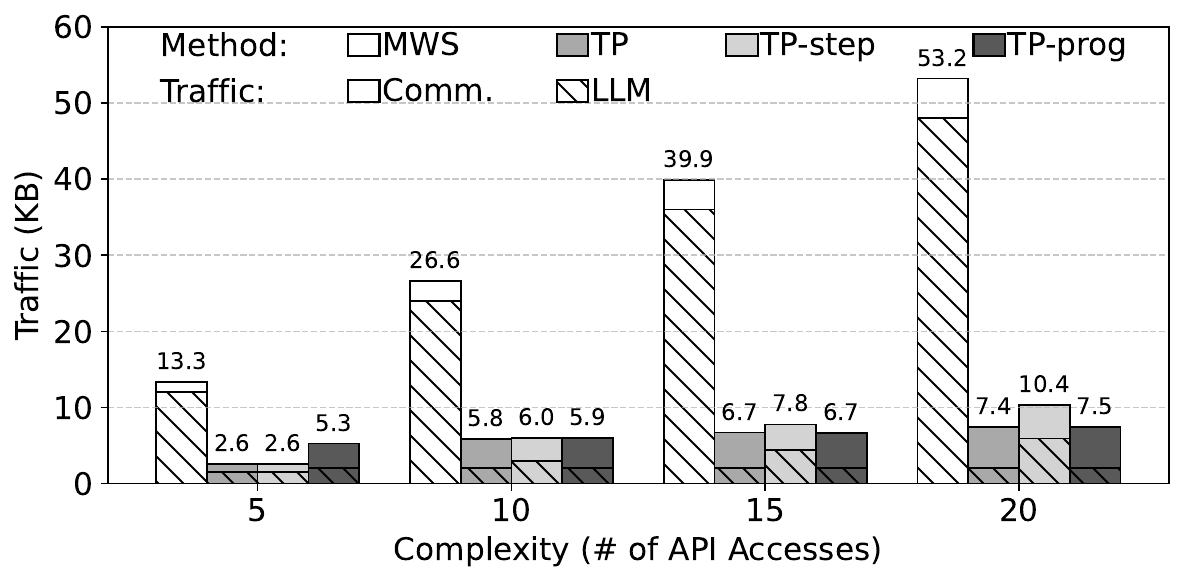}
\vspace{-2em}
\caption{Impact of workflow complexity on traffic volume.}
\label{fig:complexity_traffic}
\end{center}
\end{figure}

As shown in \autoref{fig:complexity_latency}, MWS scales roughly linearly with $N$ due to $N$ rounds of RTT and reasoning.
\sysname in stepwise mode improves modestly by reducing per-step decision overhead with intent-structured guidance, but still pays $N$ rounds.
In contrast, \sysname-prog remains comparatively stable with $N$ because it pays a one-time build cost and executes the interaction logic server-side.
The full \sysname achieves the best of both by switching between modes, validating the profile-driven policy.

As shown in \autoref{fig:complexity_traffic}, client-side traffic in MWS grows with $N$ due to repeated prompts and tool-context transmission.
\sysname-prog maintains low traffic for larger $N$ by avoiding multi-round prompting.
By switching between stepwise and program modes, \sysname minimizes client-side traffic across complexity conditions.

Moreover, we break down the overhead of \sysname-prog in \autoref{ap:overhead}.

\section{Related Work}

\paraf{Interfaces and representations for agent tool use.}
LLM-based agents increasingly solve tasks by invoking tools and APIs~\cite{shen2025shortcutsbench,liu2026roga,shen2025ragsynthsyntheticdatarobust,du2024anytool,song2025beyond,schick2023toolformer,qin2024toolllm}.
The dominant interface remains stepwise endpoints, where intent is implicit in a sequence of endpoint calls interleaved with multi-round reasoning.
This makes control flow and intermediate bindings an emergent property of the agent policy, inflating network turns and reasoning as workflows lengthen.
In contrast, \sysname makes intent explicit as an executable tool program with a constrained surface and effect types, enabling compilation, inspection, enforcement, and optimization at the interface boundary.

\paraf{From expressive queries to executable interaction logic.}
Web APIs evolved from RESTful endpoints~\cite{fielding2000architectural} to flexible query interfaces such as GraphQL~\cite{graphql}, reducing over/under-fetching by letting clients shape responses.
However, GraphQL is not designed to encode procedural interaction logic (loops, conditionals, retries) as a single interface object~\cite{graphql-indirect,hartig2018semantics,graphql-loop}.
Recent systems execute user-supplied logic near the service; e.g., ORFA~\cite{gu2025orfa} uses Wasm modules as a Turing-complete query language.
\sysname instead focuses on agentic workflows where the core bottlenecks are executability and side effects under re-execution, using LLM-synthesized, repairable tool programs with an effect-typed contract and retry-safe semantics under iterative repair.

% \draft{
\paraf{Programmatic agents and in-situ execution.}
Code-based agent methods such as CodeAct~\cite{wang2024codeact} use executable code as an agent action space, and workflow-generation methods such as AFlow~\cite{zhang2025aflow} search over code-represented agent workflows.
\sysname is complementary but targets a different boundary: the submitted program is an effect-typed service interface object, not only an internal reasoning/action representation, and the runtime must enforce replay-safe \text{WRITE} behavior across repair and re-execution.
The design is also analogous to in-situ programmable execution systems such as eBPF~\cite{hoiland2018express}, which move logic closer to the execution boundary to reduce repeated control transfers.
Unlike eBPF's pre-verified kernel packet programs, \sysname handles dynamically generated tool programs over web-service endpoints with explicit effect annotations, sandboxed execution, and fail-closed fallback.
%}

\paraf{Wasm sandboxing as a substrate.}
WebAssembly (Wasm) provides a portable sandbox with near-native performance~\cite{liu2025webanns,yan2021understanding,liu2025webassembly} and is widely adopted beyond browsers, including serverless and edge settings~\cite{hoque2023webassembly,menetrey2022webassembly,kjorveziroski2023webassembly,gackstatter2022pushing,zhang2024researchwebassemblyruntimessurvey}.
\sysname uses Wasm as a supporting substrate to securely execute the constrained tool-program surface.

\section{Conclusion}

We introduced \sysname, which makes \emph{tool programs} a first-class agent-facing service interface. By compiling effect-typed programs with projection and bounded repair, \sysname consolidates multi-step workflows into fewer turns while ensuring retry-safe execution via effect-aware replay. A profile-driven policy selects program execution only when it is predicted to reduce end-to-end cost. Across complex workflows within realistic applications, \sysname reduces latency and client-side traffic, with gains growing under higher RTT and increased workflow complexity. We believe \sysname is a practical step toward executable and effect-safe representations of tool intent for the agentic web.

\section*{Acknowledgements}
This work was supported by National Natural Science Foundation of China under the grant number 62595734, the Key Laboratory of High Confidence Software Technologies (Peking University), the Ministry of Education, and the Center for Data Space Technology and System, Peking University.

\section*{Impact Statement}

This paper presents work whose goal is to advance the field of Machine Learning. There are many potential societal consequences of our work, none of which we feel must be specifically highlighted here.

\bibliography{main}

@inproceedings{zhou2024webarena,
  title={Webarena: A realistic web environment for building autonomous agents},
  author={Zhou, Shuyan and Xu, Frank F and Zhu, Hao and Zhou, Xuhui and Lo, Robert and Sridhar, Abishek and Cheng, Xianyi and Ou, Tianyue and Bisk, Yonatan and Fried, Daniel and others},
  booktitle={Proceedings of the 2024 International Conference on Learning Representations (ICLR 2024)},
  pages={15585--15606},
  year={2024}
}

@article{deng2023mind2web,
  title={Mind2web: Towards a generalist agent for the web},
  author={Deng, Xiang and Gu, Yu and Zheng, Boyuan and Chen, Shijie and Stevens, Sam and Wang, Boshi and Sun, Huan and Su, Yu},
  journal={Proceedings of the Advances in Neural Information Processing Systems (NeurIPS 2023)},
  pages={28091--28114},
  year={2023}
}

@article{yao2022webshop,
  title={Webshop: Towards scalable real-world web interaction with grounded language agents},
  author={Yao, Shunyu and Chen, Howard and Yang, John and Narasimhan, Karthik},
  journal={Proceedings of the Advances in Neural Information Processing Systems (NeurIPS 2022)},
  pages={20744--20757},
  year={2022}
}

@inproceedings{liu2024agentbench,
  title={Agentbench: Evaluating llms as agents},
  author={Liu, Xiao and Yu, Hao and Zhang, Hanchen and Xu, Yifan and Lei, Xuanyu and Lai, Hanyu and Gu, Yu and Ding, Hangliang and Men, Kaiwen and Yang, Kejuan and others},
  booktitle={Proceedings of the International Conference on Learning Representations (ICLR 2024)},
  pages={52989--53046},
  year={2024}
}

@inproceedings{qin2024toolllm,
  title={Toolllm: Facilitating large language models to master 16000+ real-world apis},
  author={Qin, Yujia and Liang, Shihao and Ye, Yining and Zhu, Kunlun and Yan, Lan and Lu, Yaxi and Lin, Yankai and Cong, Xin and Tang, Xiangru and Qian, Bill and others},
  booktitle={Proceedings of the International Conference on Learning Representations (ICLR 2024)},
  pages={9695--9717},
  year={2024}
}

@article{schick2023toolformer,
  title={Toolformer: Language models can teach themselves to use tools},
  author={Schick, Timo and Dwivedi-Yu, Jane and Dess{\`\i}, Roberto and Raileanu, Roberta and Lomeli, Maria and Hambro, Eric and Zettlemoyer, Luke and Cancedda, Nicola and Scialom, Thomas},
  journal={Proceedings of the Advances in neural information processing systems (NeurIPS 2023)},
  pages={68539--68551},
  year={2023}
}

@inproceedings{yao2023react,
  title={ReAct: Synergizing Reasoning and Acting in Language Models},
  author={Yao, Shunyu and Zhao, Jeffrey and Yu, Dian and Du, Nan and Shafran, Izhak and Narasimhan, Karthik and Cao, Yuan},
  booktitle={Proceedings of the International Conference on Learning Representations (ICLR 2023)},
  year={2023}
}

@inproceedings{liu2026roga,
title={{ROGA}: Scaling Generalist Agents for Office Productivity Tasks via Tool Generation},
author={Liu, Mugeng and Ma, Xiaojun and Xie, Yuhang and Chen, Qin and Liu, Xuanzhe and Ma, Yun},
booktitle={Proceedings of the Fourteenth International Conference on Learning Representations (ICLR 2026)},
year={2026}
}

@misc{graphql-indirect,
  author = {Stack-Overflow},
  title = {{GraphQL}: can you mutate the results of a query?},
  howpublished = {\url{https://stackoverflow.com/questions/52330018/graphql-can-you-mutate-the-results-of-a-query}},
  note = {Accessed: 2025-07-28},
  year = {2018}
}

@inproceedings{
shen2025shortcutsbench,
title={ShortcutsBench: A Large-Scale Real-world Benchmark for API-based Agents},
author={Haiyang Shen and Yue Li and Desong Meng and Dongqi Cai and Sheng Qi and Li Zhang and Mengwei Xu and Yun Ma},
booktitle={Proceedings of the Thirteenth International Conference on Learning Representations (ICLR 2025)},
year={2025}
}

@article{liu2025webassembly,
author = {Liu, Mugeng and Shen, Haiyang and Zhang, Yixuan and Mei, Hong and Ma, Yun},
title = {WebAssembly for Container Runtime: Are We There Yet?},
year = {2025},
journal={ACM Transactions on Software Engineering and Methodology (TOSEM 2025)},
pages = {22},
}

@inproceedings{du2024anytool,
  title={AnyTool: self-reflective, hierarchical agents for large-scale API calls},
  author={Du, Yu and Wei, Fangyun and Zhang, Hongyang},
  booktitle={Proceedings of the 41st International Conference on Machine Learning (ICML 2024)},
  pages={11812--11829},
  year={2024}
}

@inproceedings{song2025beyond,
  title={Beyond browsing: Api-based web agents},
  author={Song, Yueqi and Xu, Frank F and Zhou, Shuyan and Neubig, Graham},
  booktitle={Findings of the Association for Computational Linguistics: ACL 2025},
  pages={11066--11085},
  year={2025}
}

@inproceedings{liu2025webanns,
author = {Liu, Mugeng and Zhong, Siqi and Yang, Qi and Han, Yudong and Liu, Xuanzhe and Ma, Yun},
title = {WebANNS: Fast and Efficient Approximate Nearest Neighbor Search in Web Browsers},
year = {2025},
booktitle = {Proceedings of the 48th International ACM SIGIR Conference on Research and Development in Information Retrieval (SIGIR 2025)},
pages = {2483–2492}
}

@misc{shen2025ragsynthsyntheticdatarobust,
      title={{RAGSynth}: Synthetic Data for Robust and Faithful RAG Component Optimization}, 
      author={Haiyang Shen and Hang Yan and Zhongshi Xing and Mugeng Liu and Yue Li and Zhiyang Chen and Yuxiang Wang and Jiuzheng Wang and Yun Ma},
      year={2025},
      eprint={2505.10989},
      archivePrefix={arXiv},
}

@misc{graphql-loop,
  author = {Stack-Overflow},
  title = {Graphql loop through array and get all results. Stack Overflow.},
  howpublished = {\url{https://stackoverflow.com/questions/48321689/graphql-loop-through-array-and-get-all-results}},
  note = {Accessed: 2025-07-28},
  year = {2018}
}

@book{fielding2000architectural,
  title={Architectural styles and the design of network-based software architectures},
  author={Fielding, Roy Thomas},
  year={2000},
  publisher={University of California, Irvine}
}

@misc{graphql,
  author = {GraphQL},
  title = {The query language for modern {APIs}.},
  howpublished = {\url{https://graphql.org/}},
  note = {Accessed: 2025-07-28},
  year = {2015}
}

@misc{qwen3max,
    title = {{Qwen3-Max}: Just Scale it},
    author = {Qwen-Team},
    month = {September},
    year = {2025}
}

@inproceedings{gu2025orfa,
  title={{ORFA}: Exploring {WebAssembly} as a Turing Complete Query Language for Web {APIs}},
  author={Gu, Yuhao and Chen, Chunyu and Du, Jiangsu and Zhang, Xiaoxi and Zhang, Xianwei},
  booktitle={Proceedings of the ACM on Web Conference 2025 (WWW 2025)},
  pages={1856--1865},
  year={2025}
}

@inproceedings{gackstatter2022pushing,
  title={Pushing serverless to the edge with {WebAssembly} runtimes},
  author={Gackstatter, Philipp and Frangoudis, Pantelis A and Dustdar, Schahram},
  booktitle={Proceedings of the 22nd IEEE International Symposium on Cluster, Cloud and Internet Computing (CCGrid 2022)},
  pages={140--149},
  year={2022}
}

@article{kjorveziroski2023webassembly,
  title={{WebAssembly} as an enabler for next generation serverless computing},
  author={Kjorveziroski, Vojdan and Filiposka, Sonja},
  journal={Journal of Grid Computing},
  pages={34},
  year={2023}
}

@inproceedings{menetrey2022webassembly,
  title={{WebAssembly} as a common layer for the cloud-edge continuum},
  author={M{\'e}n{\'e}trey, J{\"a}mes and Pasin, Marcelo and Felber, Pascal and Schiavoni, Valerio},
  booktitle={Proceedings of the 2nd Workshop on Flexible Resource and Application Management on the Edge},
  pages={3--8},
  year={2022}
}

@article{hoque2023webassembly,
  title={{WebAssembly} for edge computing: Potential and challenges},
  author={Hoque, Mohammed Nurul and Harras, Khaled A},
  journal={IEEE Communications Standards Magazine},
  pages={68--73},
  year={2023}
}

@article{zhang2024researchwebassemblyruntimessurvey,
  title={Research on webassembly runtimes: A survey},
  author={Zhang, Yixuan and Liu, Mugeng and Wang, Haoyu and Ma, Yun and Huang, Gang and Liu, Xuanzhe},
  journal={ACM Transactions on Software Engineering and Methodology (TOSEM 2025)},
  pages={1--47},
  year={2025},
}

@inproceedings{yan2021understanding,
  title={Understanding the performance of webassembly applications},
  author={Yan, Yutian and Tu, Tengfei and Zhao, Lijian and Zhou, Yuchen and Wang, Weihang},
  booktitle={Proceedings of the 21st ACM Internet Measurement Conference},
  pages={533--549},
  year={2021}
}

@inproceedings{hartig2018semantics,
  title={Semantics and complexity of GraphQL},
  author={Hartig, Olaf and P{\'e}rez, Jorge},
  booktitle={Proceedings of the 2018 World Wide Web Conference (WWW 2018)},
  pages={1155--1164},
  year={2018}
}

@inproceedings{wang2024codeact,
  title={Executable code actions elicit better {LLM} agents},
  author={Wang, Xingyao and Chen, Yangyi and Yuan, Lifan and Zhang, Yizhe and Li, Yunzhu and Peng, Hao and Ji, Heng},
  booktitle={Proceedings of the Forty-first International Conference on Machine Learning (ICML 2024)},
  year={2024}
}

@inproceedings{zhang2025aflow,
  title={Aflow: Automating agentic workflow generation},
  author={Zhang, Jiayi and Xiang, Jinyu and Yu, Zhaoyang and Teng, Fengwei and Chen, Xionghui and Chen, Jiaqi and Zhuge, Mingchen and Cheng, Xin and Hong, Sirui and Wang, Jinlin and others},
  booktitle={Proceedings of the International Conference on Learning Representations (ICLR 2025)},
  pages={34040--34077},
  year={2025}
}

@inproceedings{hoiland2018express,
  title={The express data path: Fast programmable packet processing in the operating system kernel},
  author={Hoiland-Jorgensen, Toke and Brouer, Jesper Dangaard and Borkmann, Daniel and Fastabend, John and Herbert, Tom and Ahern, David and Miller, David},
  booktitle={Proceedings of the 14th international conference on emerging networking experiments and technologies},
  pages={54--66},
  year={2018}
}
\bibliographystyle{icml2026}

% \newpage
\appendix
% \onecolumn

\appendix

\newpage
\section{Implementation Details}
\label{app:impl}

This appendix provides additional details of \sysname's Wasm-based runtime, including the host interface, replay state management, and policy instrumentation.

\subsection{Runtime Architecture}

\sysname is a client-server system.
The client produces a candidate tool program $P$ in Rust and an intent-instance identifier, and then either (i) submits $P$ to the server for program-mode execution or (ii) performs stepwise calling.
In program mode, the server executes the following stages.
(1) \emph{Projection} $\Pi(\cdot)$ rewrites $P$ into the constrained interface-program surface.
(2) \emph{Compilation/repair} utilizes Rustc to compile the projected Rust program and performs bounded in-place repair using compiler diagnostics and runtime traces.
(3) \emph{Sandboxed execution} runs the resulting module while intercepting every dynamic \textsc{Call} to enforce program order and effect-aware replay.
If any stage exceeds budgets or violates constraints, the server triggers safe fallback to stepwise calling and returns diagnostics.

\subsection{Wasm Runtime Configuration and Host Interface}

The overall Wasm generation pipeline is shown in \autoref{fig:wasm}.
\sysname compiles the LLM-generated Rust program into a Wasm module and executes it using Wasmtime.
The module runs with no ambient authority. It cannot issue network requests, access the filesystem, or load dynamic libraries.
All external interaction is mediated through a minimal set of host imports.

\begin{figure}[h!]
\begin{center}
\includegraphics[width=0.49\textwidth]{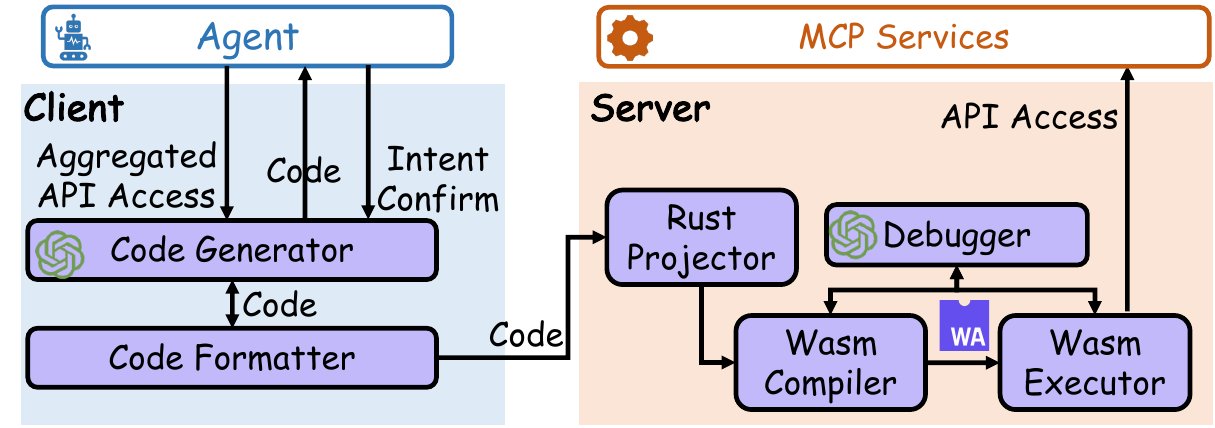}
\caption{Wasm generation pipeline of \sysname.}
\label{fig:wasm}
\end{center}
\end{figure}

\textbf{Unified call stub.}
The only capability needed by tool programs is the unified stub $\textsc{Call}(e,a,\mathsf{eff})$.
We implement this as a host-imported function that takes
(i) an endpoint identifier $e$ (interned string or integer id),
(ii) a serialized argument object $a$ (e.g., JSON bytes), and
(iii) an effect tag $\mathsf{eff}\in\{\text{READ},\text{WRITE}\}$.
The host returns a serialized result object or a typed error.
This import is the enforcement point. The mediator can log calls, enforce program order, check effect labels, and apply replay semantics without requiring the module to embed networking logic.

\textbf{Deterministic projection target.}
Projection $\Pi(\cdot)$ targets a stable surface that compiles deterministically to Wasm.
In particular, it rewrites all external interactions to the stub, rejects disallowed constructs (exceptions/threads/dynamic linking/unsafe features), and ensures effect annotations are explicit at each call site.
This makes compiler diagnostics and runtime traces comparable across repair iterations.

\subsection{Effect Mediation and Replay State}

Effect-aware replay is implemented inside the mediator that handles every dynamic \textsc{Call} from the Wasm module.
For each intent instance, the mediator maintains
(i) a history log $\mathcal{H}$ of completed \text{WRITE} calls from prior executions, and
(ii) a working log $\mathcal{W}$ for completed \text{WRITE} calls in the current execution.
Each entry stores $(e,a,o)$ and a per-run \texttt{used} flag.

\textbf{Re-execution protocol.}
Before each repair-driven re-execution, the mediator archives $\mathcal{W}$ into $\mathcal{H}$, clears $\mathcal{W}$, and resets all \texttt{used} flags.
During execution, \text{READ} calls are always forwarded to the service.
For a \text{WRITE} call with parameters $(e,a)$, the mediator matches it to the earliest unused entry in $\mathcal{H}$; if matched, it returns the cached outcome without emitting the call, otherwise it emits once and appends the outcome to $\mathcal{W}$.

\textbf{Fail-closed conditions.}
Replay is enabled only under the replay discipline.
If a repaired program attempts to revise a committed \text{WRITE} prefix (by changing arguments or relative order), replay is disabled, and the system falls back to stepwise calling with diagnostics.

% \subsection{Profiling and Policy Instrumentation}

% To support the profile-driven consolidation policy, \sysname instruments:
% (i) $T_{\textsc{rtt}}$ (client--service RTT),
% (ii) $T_{\textsc{dec}}$ (per-step client-side decision overhead in stepwise mode), and
% (iii) $T_{\textsc{build}}$ (program-mode build time including synthesis, projection, compilation, and bounded repair).
% The server maintains moving averages $\overline{T_{\textsc{rtt}}}$ and $\overline{T_{\textsc{build}}}$, and the client maintains $\overline{T_{\textsc{dec}}}$ in stepwise mode. These are combined to evaluate \autoref{eq:benefit}.
% The call-count estimate $N$ is derived from the synthesized structure when loop bounds are available, or from recent similar intents and workflows.

% \textbf{Auditing mode selection.}
% For each run, \sysname records $(\Delta T, N, \overline{T_{\textsc{rtt}}}, \overline{T_{\textsc{dec}}}, \overline{T_{\textsc{build}}})$ alongside the selected mode and outcome (success/fallback).
% This supports post-hoc analysis of mode-selection errors and helps tune bootstrapping thresholds under cold start.

\section{Experimental Details}

\subsection{Application Benchmarks}
\label{ap:benchmarks}

Details of employed applications are listed below.
\begin{itemize}[left=0pt,noitemsep,topsep=0pt,partopsep=0pt]
\item \textbf{Memos}\footnote{\url{https://github.com/usememos/memos}} (56k stars): A lightweight, self-hosted knowledge management platform in Go, exposing OpenAPI-compliant REST interfaces.
\item \textbf{Directus}\footnote{\url{https://github.com/directus/directus}} (34k stars): An API layer providing REST and GraphQL endpoints over SQL backends.
\item \textbf{MinIO}\footnote{\url{https://github.com/minio/minio}} (60k stars): An S3-compatible, high-performance object storage system.
\end{itemize}

\subsection{Constructed Workflows}
\label{ap:workflow}

\autoref{tab:workflow} shows the constructed procedural agentic workflows used in our experiments.

\begin{table}[h!]
\centering
\caption{Composed procedural workflows over fixed endpoints.}
\label{tab:workflow}
\begin{tabular}{p{1cm}<{\raggedright}|p{3cm}<{\raggedright}p{3cm}<{\raggedright}}
\toprule
\textbf{App} & \textbf{Read-Only} & \textbf{Read-Write} \\
&\textbf{Workflow (.r)}& \textbf{Workflow (.w)}\\
\midrule
Memos & Get memos of $N$ users (user\_id=1,2,\ldots,$N$). $N$ is the user number. & Change the visibility of $N$ memos. $N$ is the memo number. \\
\midrule
Directus & Get detailed information of $N$ articles (id=1,2,\ldots,$N$). $N$ is the article number. & Modify author information of $N$ draft articles. $N$ is the number of articles. \\
\midrule
MinIO & Download $N$ files. $N$ is the file number. & Upload $N$ files. $N$ is the file number. \\
\bottomrule
\end{tabular}
\end{table}

\subsection{Supplemental Realistic Workflows and Results}
\label{ap:supplemental_workflows}

The supplemental realistic benchmarks were added to measure \sysname under complex situations. They keep the same MCP-style endpoint setting but require runtime binding, branch-dependent control flow, state consistency across multiple records, and cross-service side effects.

\begin{center}
\footnotesize
\captionof{table}{Supplemental realistic workflows.}
\label{tab:supplemental_workflows}
\resizebox{\linewidth}{!}{
\begin{tabular}{p{0.16\linewidth}p{0.34\linewidth}p{0.40\linewidth}}
\toprule
\textbf{Bench} & \textbf{Core challenge} & \textbf{Representative task} \\
\midrule
cbench1 & Nondeterministic retrieval & Find customer Mei Patel by ZIP 76165, identify the correct pending order at runtime, and update that order to contain only item 1096508426 with total 55.0. \\
cbench2 & Loop, branching, and coordinated writes & Update all pending orders and the customer's primary address; if order \#W4082615 has not shipped, replace its items and set the total to 25.0. \\
cbench3 & Non-idempotent side effects & For delivered order \#D12345, change status to exchange\_requested, replace its items, and create an exchange log recording old and new items. \\
cbench4 & Cross-service branching & Upload five local documents to MinIO, inspect object text, route objects across prefixes, create follow-up memos when required, then clean up generated files and memos. \\
\bottomrule
\end{tabular}}
\end{center}

\paraf{Supplemental benchmark results.}
Across cbench1--cbench3, \sysname reduces average latency from 30.16s to 17.91s (40.6\%), improves task accuracy from 0.60 to 0.93, and lowers client-side LLM latency from 14.98s to 7.08s.
For cbench4, \sysname reduces latency from 52.68s to 24.54s (53.4\%), cuts client-side traffic by 96.1\%, and improves accuracy from 0.20 to 0.80.
The policy selects program mode in 12/15 cold-start runs and in all 12/12 warm-start runs after profiles are available; no fallback is observed once program mode is selected in these measurements.

\paraf{Replay and model comparison.}
In a 15-run no-replay ablation on cbench1--cbench3, disabling replay increases average latency from 17.92s to 21.45s and fallback from 0/15 to 3/15.
% Across six replay events, replay handling takes 0.89s on average, indicating that the main penalty without replay comes from conservative fallback rather than mediation overhead.
On cbench2, gpt-5.1 and gemini-3-flash-preview each reach 100\% success rate with no observed compilation failure or fallback, while qwen3-coder-flash reaches 80\%, indicating that compilation failures are primarily a model-capability bottleneck.

\begin{center}
\footnotesize
\captionof{table}{Extended RTT sweep on \texttt{Memos.w} with $N=8$.}
\label{tab:rtt_sweep}
\resizebox{\linewidth}{!}{
\begin{tabular}{rrrr}
\toprule
\textbf{RTT} & \textbf{\sysname-step} & \textbf{\sysname-prog} & \textbf{\sysname policy} \\
\midrule
1ms & 14.50s & 15.51s & 14.54s \\
10ms & 14.57s & 15.52s & 14.57s \\
100ms & 15.30s & 15.60s & 15.43s \\
1000ms & 22.48s & 16.50s & 16.22s \\
2000ms & 30.50s & 17.51s & 17.64s \\
\bottomrule
\end{tabular}}
\end{center}
This sweep makes the profile-driven switch explicit: stepwise execution is preferable at low RTT, while program execution dominates once RTT becomes the main cost.

\begin{center}
\footnotesize
\captionof{table}{Language comparison on cbench2. Rust is an implementation choice rather than a conceptual requirement.}
\label{tab:language_comparison}
\resizebox{\linewidth}{!}{
\begin{tabular}{lrrl}
\toprule
\textbf{Language} & \textbf{Avg. seconds} & \textbf{Accuracy} & \textbf{Note} \\
\midrule
Rust & 19.4052 & 0.80 & Mature direct-to-Wasm frontend in our setting. \\
Go & 50.2153 & 0.40 & More brittle direct-to-Wasm compilation in our setting. \\
Lua & N/A & N/A & No practical direct-to-Wasm path for this use case. \\
\bottomrule
\end{tabular}}
\end{center}
These results support using Rust in the prototype because it integrates cleanly with Wasmtime and provides actionable compiler diagnostics, while a lighter DSL/IR remains a promising future direction.

Generated tool programs remain compact in the supplemental runs: sampled programs have a median of 70 non-empty lines of code, 4 MCP tool calls, and 1 branch/loop keyword hit; the shortest and longest samples are 43 and 102 non-empty lines. The 102-line sample includes 6 tool calls (3 \text{READ}, 3 \text{WRITE}) plus a conditional branch, with explicit effect annotations at call sites.

\subsection{Environments}
\label{ap:env}

We run servers on AWS \texttt{c7i-flex.large} instances (Intel Xeon Sapphire Rapids @2.40GHz, 2 vCPUs, 4GB RAM, up to 12.5Gbps).
Servers are hosted in Sydney, London, and San Francisco.
The client runs in Beijing on a machine with an Intel Core i9-14900HX CPU, 16GB RAM, and a gigabit network connection.
\sysname is implemented in Python, utilizing the Wasmtime\footnote{\url{https://github.com/bytecodealliance/wasmtime}} for WebAssembly execution. 
\sysname uses Rust\footnote{\url{https://rust-lang.org/}} as the high-level language for Wasm compilation. 
We use Qwen3-Max-Instruct~\citep{qwen3max} for program synthesis and repair prompting on the client.
Unless otherwise specified, results are averaged over five runs.

\section{Overhead Breakdown}
\label{ap:overhead}

\begin{table}[ht!]
\centering
\caption{Breakdown of tool-program build pipeline latency (seconds and percentage of total end-to-end latency).}
\label{tab:overhead}
\resizebox{0.49\textwidth}{!}{
\begin{tabular}{lc|rrrr}
\toprule
\textbf{App} &\textbf{$N$} & \textbf{Program} & \textbf{Compila-} & \textbf{Sandbox} & \textbf{Service} \\
& & \textbf{Synthesis} & \textbf{tion} & \textbf{Execution} & \textbf{Execution} \\
\midrule
\multirow{4}{*}{Memos.r}
& 5 & 7.31 (41.9\%) & 5.39 (30.9\%) & 2.44 (14.0\%) & 2.06 (11.8\%) \\
& 10 & 7.32 (40.9\%) & 5.43 (30.3\%) & 2.46 (13.7\%) & 2.40 (13.4\%) \\
& 15 & 7.61 (41.9\%) & 5.39 (29.7\%) & 2.45 (13.5\%) & 2.58 (14.2\%) \\
& 20 & 7.61 (41.3\%) & 5.42 (29.4\%) & 2.47 (13.4\%) & 2.81 (15.2\%) \\
\midrule
\multirow{4}{*}{Memos.w}
& 5 & 7.31 (42.4\%) & 5.17 (30.0\%) & 2.47 (14.3\%) & 2.16 (12.5\%) \\
& 10 & 7.51 (42.0\%) & 5.21 (29.2\%) & 2.48 (13.9\%) & 2.20 (12.3\%) \\
& 15 & 7.62 (42.3\%) & 5.37 (29.8\%) & 2.47 (13.7\%) & 2.38 (13.2\%) \\
& 20 & 7.86 (41.3\%) & 5.45 (28.6\%) & 2.49 (13.1\%) & 3.09 (16.2\%) \\
\bottomrule
\end{tabular}
}
\end{table}

We break down the overhead of program mode by measuring \sysname-prog on \texttt{Memos.r} and \texttt{Memos.w} with $N=5,10,15,20$ (server in Sydney, client in Beijing).
\autoref{tab:overhead} shows that program synthesis dominates the one-time build cost, while compilation and sandbox execution remain stable across $N$.
This suggests that model specialization for the constrained interface-program surface (or more efficient synthesis/repair prompting) could further reduce overhead.
Importantly, sandbox execution includes effect mediation for \text{WRITE} calls, yet contributes a relatively small fraction of end-to-end time, indicating that retry safety can be enforced with modest overhead in practice.

%%%%%%%%%%%%%%%%%%%%%%%%%%%%%%%%%%%%%%%%%%%%%%%%%%%%%%%%%%%%%%%%%%%%%%%%%%%%%%%
%%%%%%%%%%%%%%%%%%%%%%%%%%%%%%%%%%%%%%%%%%%%%%%%%%%%%%%%%%%%%%%%%%%%%%%%%%%%%%%
% APPENDIX
%%%%%%%%%%%%%%%%%%%%%%%%%%%%%%%%%%%%%%%%%%%%%%%%%%%%%%%%%%%%%%%%%%%%%%%%%%%%%%%
%%%%%%%%%%%%%%%%%%%%%%%%%%%%%%%%%%%%%%%%%%%%%%%%%%%%%%%%%%%%%%%%%%%%%%%%%%%%%%%
% \newpage
% \appendix
% \onecolumn

% You can have as much text here as you want. The main body must be at most $8$
% pages long. For the final version, one more page can be added. If you want, you
% can use an appendix like this one.

% The $\mathtt{\backslash onecolumn}$ command above can be kept in place if you
% prefer a one-column appendix, or can be removed if you prefer a two-column
% appendix.  Apart from this possible change, the style (font size, spacing,
% margins, page numbering, etc.) should be kept the same as the main body.
%%%%%%%%%%%%%%%%%%%%%%%%%%%%%%%%%%%%%%%%%%%%%%%%%%%%%%%%%%%%%%%%%%%%%%%%%%%%%%%
%%%%%%%%%%%%%%%%%%%%%%%%%%%%%%%%%%%%%%%%%%%%%%%%%%%%%%%%%%%%%%%%%%%%%%%%%%%%%%%

\end{document}